\documentclass{article}
\usepackage{graphicx} 
\usepackage[a-2u]{pdfx}

\usepackage[utf8]{inputenc}

\usepackage{lmodern}
\usepackage{amsmath}        
\usepackage{amsfonts}       
\usepackage{amsthm}         
\usepackage{bbding}         
\usepackage{bm}             
\usepackage{graphicx}       
\usepackage{fancyvrb}       
\usepackage[nottoc]{tocbibind} 
\usepackage{dcolumn}        
\usepackage{booktabs}       
\usepackage{paralist}       
\usepackage{mathtools}
\usepackage{svg}
\usepackage{algorithm}
\usepackage{algpseudocode}
\usepackage{comment}
\usepackage{enumitem}
\usepackage{float}
\usepackage{etoolbox}
\makeatletter
\pretocmd\@maketitle
  {\def\@makefnmark{\hbox{\@textsuperscript{\normalfont\@thefnmark}}}}
  {}{\FAIL}
\makeatother

\DeclarePairedDelimiter{\set}{\{}{\}}
\DeclarePairedDelimiter{\floor}{\lfloor}{\rfloor}

\newtheorem{dfn}{Definition}[section]
\newtheorem{lemma}[dfn]{Lemma}
\newtheorem{thm}[dfn]{Theorem}
\newtheorem{prop}[dfn]{Proposition}
\newtheorem{obs}[dfn]{Observation}
\newtheorem{cor}[dfn]{Corollary}

\title{Boolean Nearest Neighbor Language \\ in the Knowledge Compilation Map}
\author{Ond\v{r}ej \v{C}epek\thanks{Department of Theoretical Computer Science and Mathematical Logic, Charles University, Malostransk\'e n\'am. 25, 11800 Praha 1, Czech Republic. (ondrej.cepek@mff.cuni.cz)} and Jelena Gli\v{s}i\'{c}\thanks{Department of Applied Mathematics, Charles University, Malostransk\'e n\'am. 25, 11800 Praha 1, Czech Republic. (glisic@kam.mff.cuni.cz)}}

\begin{document}

\maketitle

\begin{abstract}
    The Boolean Nearest Neighbor (BNN) representation of Boolean functions was recently introduced by Hajnal, Liu and Tur\'{a}n~\cite{NNR}. A BNN representation of $f$ is a pair $(P,N)$ of sets of Boolean vectors (called positive and negative prototypes) where $f(x)=1$ for every positive prototype $x \in P$, $f(x)=0$ for all every negative prototype $x \in N$, and the value $f(x)$ for $x \not\in P \cup N$ is determined by the type of the closest prototype. The main aim of this paper is to determine the position of the BNN language in the Knowledge Compilation Map (KCM). To this end, we derive results that compare the succinctness of the BNN language to several standard languages from KCM and determine the complexity status of most standard queries and transformations for BNN inputs.
\end{abstract}

\section{Introduction}

Boolean functions constitute a fundamental concept in computer science, which often serves as the backbone of computation tasks and decision-making processes. Their significance extends across diverse domains including circuit design, artificial intelligence and cryptography. However, as the complexity of systems increases, the need for efficient representation and manipulation of Boolean functions becomes more and more important. There are many different ways in which a Boolean function may be represented. Common representations include truth tables (TT) (where a function value is explicitly given for every binary vector), list of models (MODS), i.e. a list of binary vectors on which the function evaluates to 1, various types of Boolean formulas (including CNF and DNF representations), various types of binary decision diagrams (BDDs, FBDDs, OBDDs), negational normal forms (NNF, DNNF, d-DNNF), and Boolean circuits.

The task of transforming one of the representations of a given function $f$ into another representation of $f$ (e.g. transforming a DNF representation into an OBDD or a DNNF into a CNF) is called knowledge compilation. This task emerged as an important ingredient of modern computation, aiming to transform Boolean functions into more compact and tractable forms. This allows us to use techniques from logic, algorithms, and data structures in order to bridge the gap between the expressive power of Boolean functions and the efficiency requirements of practical applications. For a comprehensive review paper on knowledge compilation see~\cite{KCM}, where the Knowledge Compilation Map (KCM) is introduced. KCM systematically investigates different knowledge representation languages with respect to (1) their relative succinctness, (2) the complexity of common transformations, and (3) the complexity of common queries. The \emph{succinctness} of representations roughly speaking describes how large the output representation in language $B$ is with respect to the size of the input representation in language $A$ when compiling from $A$ to $B$. A precise definition of this notion will be given later in this text. Transformations include negation, conjunction, disjunction, conditioning, and forgetting. The complexity of such transformations may differ dramatically from trivial to computationally hard depending on the chosen representation language. The same is true for queries such as consistency check, validity check, clausal and sentential entailment, equivalence check, model counting, and model enumeration. 

The number of knowledge representation languages in the Knowledge Compilation Map gradually increases. In~\cite{pbc} the authors included Pseudo-Boolean constraint (PBC) and Cardinality constraint (CARD) languages into KCM by showing succinctness relations among PBC, CARD, and languages already in the map, and by proving the complexity status of all queries and transformations introduced in~\cite{KCM}. The same goal was achieved for the switch-list (SL) language in~\cite{SL}. In this paper, we aim at solving exactly the same task for Boolean Nearest Neighbor (BNN) representations introduced recently  in~\cite{NNR}. Let us denote $\mathcal{B} = \{0,1\}$ and $d_H$ the Hamming distance on $\mathcal{B}^n$.  

\begin{dfn}[\cite{NNR}]
    A \textup{Boolean nearest neighbor (BNN)} representation of a Boolean function $f$ in $n$ variables is a pair of disjoint subsets $(P,N)$ of $\mathcal{B}^n$ (called \textup{positive and negative prototypes}) such that for every $a\in\mathcal{B}^n$
    \begin{itemize}
        \item if $f(a)=1$ then there exists $b\in P$ such that for every $c\in N$ it holds that $d_H(a,b)<d_H(a,c)$,
        \item if $f(a)=0$ then there exists $b\in N$ such that for every $c\in P$ it holds that $d_H(a,b)<d_H(a,c)$.
    \end{itemize}
    The Boolean nearest neighbor complexity of $f$, $BNN(f)$, is the minimum of the sizes of the BNN representations of $f$.
\end{dfn}


BNN representations of Boolean functions were introduced in~\cite{NNR} as a special case of a more general Nearest Neighbor (NN) representations where the prototypes are not restricted to Boolean vectors but instead can be any real-valued vectors from $\mathbb{R}^n$. In this case, Euclidean distance $d_E$ is used instead of the Hamming distance $d_H$. Notice that there is in fact no difference between using Euclidean and Hamming distance when considering Boolean prototypes. This is due to the fact that for any two vectors $x,y\in\mathcal{B}^n$, $d_E(x,y)=\sqrt{d_H(x,y)}$. Thus any BNN representation of a function $f$, is also an NN representation of $f$. NN representations of Boolean functions were already studied more than thirty years ago~\cite{Wilfong91}, but there are also many recent works~\cite{BBC18,GKN18,KSB23}.

NN representations of Boolean functions are a special case of an even more general concept of NN classification problems in which prototypes can be of several (in general more than two) classes. Such classification problems were studied in the area of machine learning, learning theory, and computational geometry~\cite{LB04,KAM11,AR18}. In certain contexts, the term "anchor" is used instead of "prototype". When dealing with a nearest neighbor representations, the objective is usually to minimize the number of prototypes. 

Another way of looking at BNN representations is to view them as a special case of partially defined Boolean functions (pdBf)~\cite{CHI88,BIM98} (see also Chapter 12 in~\cite{CramaHammer} for a review of pdBf literature), with an explicit way in which the function is extended to all unclassified vectors. 

In the wide range of different knowledge representation languages BNN representations belong to the group of languages based on lists of Boolean vectors. Other languages from this group are TT (full truth table), MODS (list of models), IR (interval representations)~\cite{SGZ05,CKK08} and SLR (switch list representations)~\cite{CH17,SL}.
As we shall see in the current paper, within this group the BNN language supports the least number of standard queries and transformations.

\section{Definitions and Recent Results}

Throughout this paper, we make use of the following notation and notions:
\begin{itemize}
    \item the symbol $\mathcal{B}$ denotes the set $\set{0,1}$ and $\mathcal{B}^n$ denotes the Boolean hypercube;
    \item by $d_H(x,y)$ for $x\in\mathcal{B}^n$ we denote the Hamming distance in the Boolean hypercube, i.e. the number of coordinates in which $x$ and $y$ differ;
    \item  by $d_H(x,\mathcal{A})$ for $x,y\in\mathcal{B}^n$ and $\mathcal{A}\subseteq\mathcal{B}^n$ we denote $\min\set{d_H(x,y)\mid y\in\mathcal{A}}$;
    \item by $|x|$ for $x\in\mathcal{B}^n$ we denote the weight of $x$, i.e. the number of coordinates of $x$ which are equal to $1$;
\end{itemize}

\begin{dfn}
    A \textup{Boolean function} in $n$ variables is a function $\mathcal{B}^n\to\mathcal{B}$. We say that $x\in\mathcal{B}^n$ is \textup{positive} vector or a \textup{model} of $f$ (resp. \textup{negative} vector or a \textup{non-model} of $f$) if $f(x)=1$ (resp. $f(x)=0$).
\end{dfn}

\begin{dfn}
    We call a Boolean function $f$ in $n$ variables a \textup{symmetric function} if there exists a set $I_f\subseteq\set{0,1,\dots,n}$ such that $f(x)=1$ if and only if $|x|\in I_f$.
\end{dfn}

\begin{dfn}
    The symmetric Boolean function in $n$ variables with $I_f=\set{i \text{ is odd}\mid1\le i\le n}$ is called the \textup{parity function} and we will denote it by $PAR_n$.
\end{dfn}

\begin{dfn}
    The symmetric Boolean function in $n$ variables with $I_f=\set{i\mid i\ge t}$ for $1\le t\le n$ is called the \textup{threshold function} with unit weights and threshold $t$ and we will denote it by $TH_n^t$.
\end{dfn}

\begin{dfn}
    The Boolean function $TH_n^{n/2}$ is called the \textup{majority function} and we will denote it by $MAJ_n$.
\end{dfn}

\begin{dfn}
    The \textup{Boolean hypercube graph} $\mathcal{B}_n = (V,E)$ is an undirected graph with vertex set $V=\mathcal{B}^n$ and edge set $E = \{(x,y) \mid d_H(x,y)=1 \}$. Finally, for a set of Boolean vectors $S \subseteq \mathcal{B}^n$ we define its neighborhood by
    \[
    \delta(S)=\set{x\in\mathcal{B}^n\mid d_H(x,S)=1}
    \]
\end{dfn}

In the following sections, we will often use the terms Boolean hypercube (denoted by $\mathcal{B}^n$) and Boolean hypercube graph (denoted by $\mathcal{B}_n$) interchangeably. In particular, a neighbor of a vector $x$ in the Boolean hypercube will be any vector that is adjacent to it in the Boolean hypercube graph. Now we present two recent results and two easy corollaries which we will frequently use in the rest of this paper.

\begin{thm}[\cite{NNR}]
\label{thm: bnn}
    The following statements hold:
    \begin{enumerate}[label=(\alph*)]
        \item $BNN(PAR_n)=2^n$.
        \item If $n$ is odd then $BNN(MAJ_n) = 2$ and if $n$ is even then $BNN(MAJ_n) \le \frac{n}{2}+2$.
        \item $BNN(TH_n^{\lceil n/3\rceil})=2^{\Omega(n)}$
    \end{enumerate}
\end{thm}

\begin{lemma}[\cite{nnc}]\label{lem: comps}
     Let $f$ be a Boolean function on $n$ variables. If the subgraph of $\mathcal{B}_n$ induced by the models of $f$ has $m$ connected components then any BNN representation of $f$ has at least one positive prototype in each connected component (and hence at least $m$ positive prototypes in total).
\end{lemma}

Clearly, a similar lemma holds by a symmetric argument also for non-models.

\begin{cor}\label{cor: neg}
     Let $f$ be a Boolean function on $n$ variables. If the subgraph of $\mathcal{B}_n$ induced by the non-models of $f$ has $m$ connected components then any BNN representation of $f$ has at least one negative prototype in each connected component (and hence at least $m$ negative prototypes in total).
\end{cor}

Another easy consequence of the previous two lemmas concerns isolated vectors. Vector $x \in \mathcal{B}^n$ is called \textup{isolated} for $f$ if all its neighbors have the opposite function value (i.e. $x$ represents a connected component of size one in the corresponding subgraph).

\begin{cor}\label{cor: pos}
    Let $f$ be a Boolean function and let $x$ be a positive (resp. negative) vector such that $x$ is isolated for $f$. Then $x$ must be a positive (resp. negative) prototype in any BNN representation of $f$.
\end{cor}

\section{Succinctness}
In this section, we establish the position of the \textbf{BNN} language in the succinctness diagram presented in KCM~\cite{KCM}. Let us start with a formal definition of this notion. 

\begin{dfn}
     Let $L_1$ and $L_2$ be two knowledge representation languages. We say that $L_1$ is \textup{at least as succinct as} $L_2$, denoted $L_1 \leq L_2$, if and only if there exists a polynomial $p$ such that for every sentence $\alpha\in L_2$, there exists an equivalent sentence $\beta\in L_1$ where $|\beta| \leq p(|\alpha|)$. Furthermore, $L_1$ is \textup{strictly more succinct than} $L_2$, denoted $L_1<L_2$, if $L_1\le L_2$ but $L_2\not\le L_1$.
\end{dfn}

\begin{figure}[h]
            \centering            \includegraphics[width=0.6\linewidth]{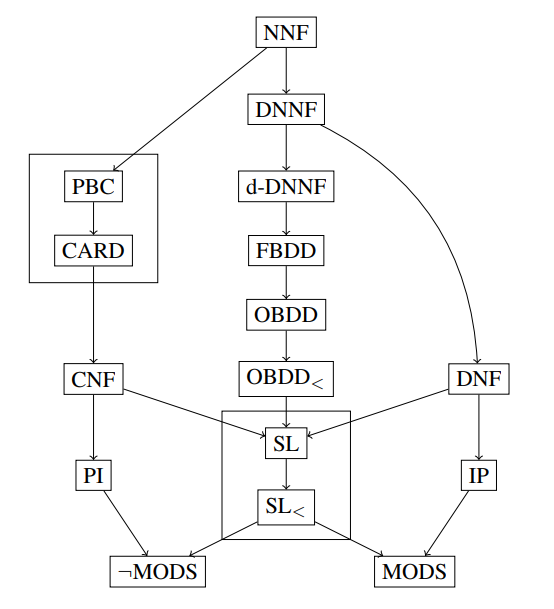}
            \caption{Diagram representing known succinctness results. The languages included are from \cite{KCM}, \cite{pbc} and \cite{SL}. An edge $L_1\to L_2$ indicates that $L_1$ is strictly more succinct than $L_2$.}.
            \label{fig: succ_old}
\end{figure}

We show a diagram representing the known succinctness results in Figure \ref{fig: succ_old}. Our goal is to prove the seven relations that connect \textbf{BNN} with standard KCM languages in the diagram in Figure~\ref{fig: succ}.

\begin{figure}[h]
    \centering
    \includegraphics[width=0.55\linewidth]{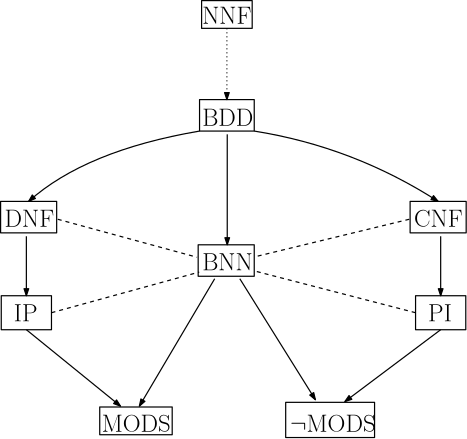}
    \caption{A solid directed edge from $L_1$ to $L_2$ indicates that $L_1$ is strictly more succinct than $L_2$. A dotted directed edge from $L_1$ to $L_2$ indicates that $L_1$ is at least as succinct as $L_2$. A dashed undirected edge between $L_1$ and $L_2$ means that $L_1$ and $L_2$ are incomparable.}
    \label{fig: succ}
\end{figure}

It should be noted that the \textbf{BDD} language is missing in the succinctness diagram in Figure \ref{fig: succ_old}, so we should also argue about the relations that concern \textbf{BDD}. The relation \textbf{NNF} $\leq$ \textbf{BDD} is trivial since
\textbf{BDD} is equivalent to a strict subset of \textbf{d-NNF}~\cite{KCM} which is in turn by definition a strict subset of \textbf{NNF}. To show the remaining two relations 
\textbf{BDD} $<$ \textbf{CNF} and \textbf{BDD} $<$ \textbf{DNF} we utilize the fact that any Boolean formula with negations at the variables, which uses only conjunction and disjunction as connectives, can be compiled into a BDD with only a linear blowup~\cite{wegBook}, and hence both \textbf{BDD} $\leq$ \textbf{CNF} and \textbf{BDD} $\leq$ \textbf{DNF} hold. This, combined with the fact that \textbf{CNF} and \textbf{DNF} are incomparable, easily yields that both inequalities are strict. We begin our study of the relations concerning the \textbf{BNN} language by showing that \textbf{BNN} is strictly less succinct than \textbf{BDD}. 

\begin{thm}\label{thm: bdd}
        \textbf{BDD} $<$ \textbf{BNN}.
\end{thm}
\begin{proof}
For the non-strict relation \textbf{BDD} $\leq$ \textbf{BNN} it suffices to show that there exists a polynomial $p$ such that for every sentence $\alpha\in\mathbf{BNN}$, there exists an equivalent sentence $\beta\in\mathbf{BDD}$ where $|\beta|\leq p(|\alpha|)$.
For any $\alpha=(P,N)$, we will construct such a binary decision diagram $\beta$. Let us assume that $P\cup N=\{p^1,\dots,p^k\}$ and let $x\in\mathcal{B}^n$. We build $\beta$ in two steps:
\begin{enumerate}
    \item We construct a gadget which for two fixed prototypes $p^i$ and $p^j$ decides which one is closer to input $x$.
    \item We put a number of these gadgets together so that the prototype closest to $x$ is found and its value outputted.
\end{enumerate}
    
We begin by building the BDD gadget. For a fixed $p^i,p^j\in P\cup N$, we construct a diagram $G_{i,j}$, as shown in Figure \ref{fig: gadget} for $n=3$. The gadget first compares each coordinate of $x$ with the corresponding coordinate of $p^i$. Thus the $n+1$ nodes on level $n+1$ of the gadget reflect the value $d_H(x,p^i)$ from $d_H(x,p^i)=0$ on the right (in this case $x=p^i$) to $d_H(x,p^i)=n$ on the left (both vectors differ in every coordinate). 

\begin{figure}[ht]
    \centering
    \includegraphics[width=0.6\linewidth]{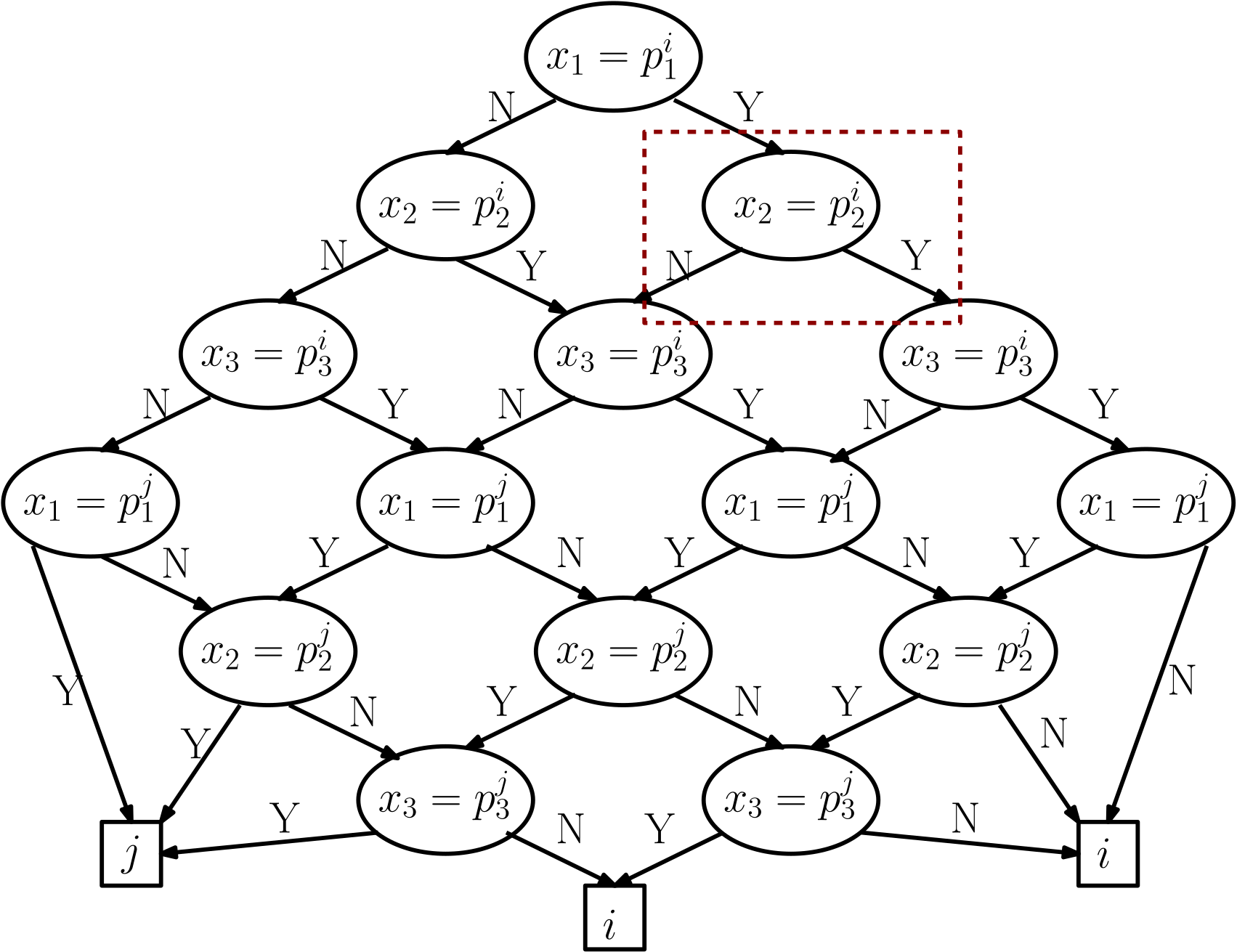}
    \caption{Gadget $G_{i,j}$ for $n=3$. Each comparison node can be replaced by a decision node, as shown in Figure \ref{fig: gadget_eq}.}
    \label{fig: gadget}
\end{figure}
   
Consider for an example the gadget in Figure \ref{fig: gadget} with $x=(1,0,1)$ and $p^i=(0,0,1)$. In the first three coordinated comparisons, we take the edges labeled $N$, $Y$ and $Y$. As only one coordinate differs, on level $4$ we know that the value of $d_H(x,p^i)$ is $1$. 
   
In the next $n$ levels (starting at level $n+1$), the gadget sequentially compares coordinates of $x$ with the corresponding coordinates of $p^j$. The comparisons stop as soon as it is decided which of $d_H(x,p^i)$ and $d_H(x,p^j)$ is smaller (e.g. $d_H(x,p^i)=n$ and $x_1 = p_1^j$ already implies $d_H(x,p^i) > d_H(x,p^j)$ without testing the remaining coordinates). As soon as the gadget determines which of the two prototypes $p^i,p^j$ is closer, it points either to the root node of a next gadget $G_{i,k}$ (if $p^i$ is closer to $x$) or $G_{j,k}$ (if $p^j$ is closer to $x$) for some $k$, or to one of the terminals $0,1$ if the closest prototype was already determined. In Figure \ref{fig: gadget} the corresponding directed edges go to the index of the closer prototype. The gadgets can break ties $d_H(x,p^i) = d_H(x,p^j)$ arbitrarily (tie is achieved at the two edges in the center of the bottom level), since the aim is to find one of the nearest prototypes, and those must all necessarily belong to the same class, by the definition of BNN.
   
Continuing with the example above, consider $p^j=(0,0,0)$. Then the edges visited in the bottom half of the gadget have labels $N$, $Y$ and $N$, we have $d_H(x,p^j)=2$, and the gadget outputs the index  $i$ of the prototype closer to $x$.

Since our aim is to construct a BDD, we must somehow convert the comparison nodes in the gadget to standard decision nodes. However, this is of course easy. Each gadget is defined for fixed prototypes, so there is an obvious way how to convert the Y and N labels on the outgoing edges from a $x_k = p_k^i$ node into $0$ and $1$ labels from a decision node on variable $x_k$, as depicted in Figure \ref{fig: gadget_eq} for $k=2$.

\begin{figure}[ht]
    \centering
    \includegraphics[scale=0.425]{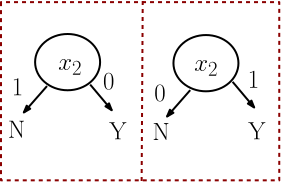}
    \caption{A decision node that corresponds to the comparison node $x_2 = p_2^i$ if $p^i_2=0$ (left) and if $p^i_2=1$ (right).}
    \label{fig: gadget_eq}
\end{figure}

    It now remains to build the output BDD $\beta$ using the gadgets. We can sequentially test pairs of prototypes until the closest one is found. We do so by making a triangle-shaped acyclic directed graph made of gadgets with $k-1$ levels. At level $i$, prototype $p^{i+1}$ is compared with every $p^j$ for $j \leq i$, and directed edges are pointed to appropriate gadgets on the next level. After level $k-1$, a closest prototype has been determined, and so the directed edges going out of these gadgets point to terminal $1$ (resp. $0$) depending on whether the found closest prototype is positive (resp. negative). The output BDD $\beta$ for a function with $k$ prototypes is shown in Figure \ref{fig: bdd}.
    \begin{figure}[h]
        \centering
        \includegraphics[width=0.6\linewidth]{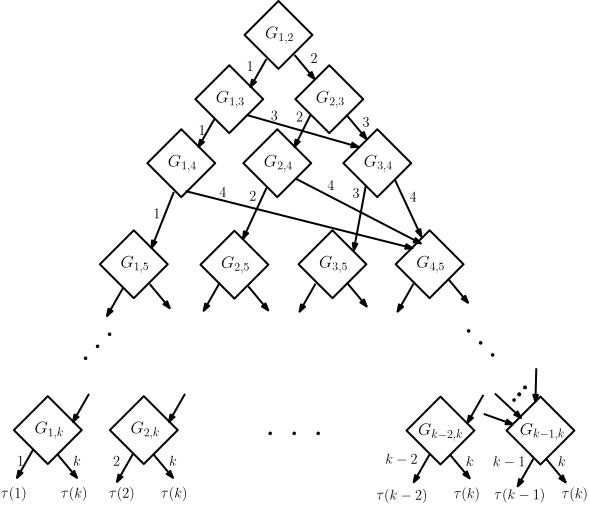}
        \caption{BDD of a function with $k$ prototypes, where $\tau(i)=1$ if prototype $p^i$ is positive and $\tau(i)=0$ if prototype $p^i$ is negative.}
        \label{fig: bdd}
    \end{figure}
        
    We now show that the constructed BDD $\beta$ is of polynomial size with respect to $|\alpha|=kn$. For vectors of length $n$, each gadget consists of $(n+1)^2-1=O(n^2)$ decision nodes. As there are $k$ prototypes in $\alpha$, BDD $\beta$ consists of $\frac{1}{2}(k-1)k=O(k^2)$ gadgets. Hence the size of the constructed BDD $\beta$ is $O(n^2k^2)=O(|\alpha|^2)$, as desired.

    This finishes the proof of \textbf{BDD} $\leq$ \textbf{BNN}.    
    To show that the inequality is strict, it suffices to consider the parity function $f$ on $n$ variables. It has a unique BNN representation with $2^n$ prototypes by Theorem \ref{thm: bnn}. On the other hand, it is well-known~\cite{wegBook} that $f$ admits a BDD representation of size $O(n)$.
\end{proof}

We now continue with proving the remaining two strict inequalities in Figure~\ref{fig: succ} which tie the \textbf{BNN} language to languages  \textbf{MODS} and $\mathbf{\lnot}$\textbf{MODS}. 

\begin{prop}
        \textbf{BNN} $\mathbf{<}$ \textbf{MODS}.
    \end{prop}
    \begin{proof}
        We first show that \textbf{BNN} $\mathbf{\leq}$ \textbf{MODS}. Consider a Boolean function $f$ with a set of models $M \subseteq \mathcal{B}^n$ where $|M|=m$. Define $$(P,N)=(M,\delta(M)).$$ We claim that $(P,N)$ is a BNN representation of $f$ with size polynomial in $mn$, which is the number of bits to store $M$.
        
        Since all models of $f$ are prototypes in $P$ it suffices to check non-models to verify that $(P,N)$ indeed represents $f$. Let $x$ be an arbitrary vector such that $f(x)=0$. If $x\in N$, then $x$ is surely classified correctly. Consider $x\not\in N$ (which implies $d_H(x,M)\ge2$) and let $z \in P$ be a positive prototype closest to $x$. Clearly, any shortest path from $x$ to $z$ must pass through some $y \in \delta(M)$ and thus through a negative prototype which is closer to $x$ than $z$ which means that $x$ is classified correctly.
        
        In the worst case, all $n$ neighbors of every model need to be picked as negative prototypes. Thus the constructed BNN $(P,N)$ has at most $m+nm$ prototypes (with $n(m+nm)$) bits) which finishes the proof of \textbf{BNN} $\mathbf{\leq}$ \textbf{MODS}.

        To see that the inequality is strict it suffices to consider 
        the constant $1$ function on $n$-dimensional vectors which has $2^n$ models but can be represented by a single positive prototype (and no negative ones).
    \end{proof}

    \begin{cor}
        \textbf{BNN} $\mathbf{<}$ $\mathbf{\lnot}$\textbf{MODS}.
    \end{cor}
    \begin{proof}
        The corollary follows by an argument analogous to the one for \textbf{MODS} where the constants $0$ and $1$ exchange their roles.
    \end{proof}
    
Now we are ready to show that \textbf{BNN} is incomparable to both \textbf{CNF} and \textbf{DNF}. 

\begin{lemma}\label{bnn_cnf}
    \textbf{BNN} $\mathbf{\not\leq}$ \textbf{CNF}.
\end{lemma}

\begin{proof}
    It suffices to show that there is a Boolean function with CNF representation of size polynomial in the number of variables, which cannot be represented by a BNN of polynomial size. We will prove that the family of functions defined by 
    \[
    f_n = \bigwedge_{i=1}^n (x_i \oplus y_i)
    \]
    (where each $f_n$ is a function in $2n$ variables $\{x_1, \ldots ,x_n, y_1, \ldots ,y_n\}$ and $\oplus$ is the XOR operator) has the desired property. Clearly, each $f_n$ has a CNF representation $F_n$ of size $O(n)$
    \[
    F_n = \bigwedge_{i=1}^n (x_i \vee y_i) \wedge (\lnot x_i \vee \lnot y_i).
    \]
    On the other hand, $f_n$ has $2^n$ models (for each $i$ we can decide whether to satisfy $x_i$ or $y_i$) and every model of $f_n$ is isolated (flipping a single bit falsifies the corresponding XOR). Therefore any BNN representation of $f_n$ has at least $2^n$ prototypes by Corollary~\ref{cor: pos}. It should be noted that this construction is a special case of the more general construction in~\cite{nnc} (Theorem 11).     
\end{proof}
\begin{lemma}\label{cnf_bnn}
    \textbf{CNF} $\mathbf{\not\leq}$ \textbf{BNN}.
\end{lemma}
\begin{proof}
    Consider the majority function $g_n$ on $n$ variables which can be represented by a BNN of size $O(n)$ by Theorem \ref{thm: bnn}. On the other hand, any CNF of this monotone function must contain at least as many clauses as is the number of maximal false vectors of $g_n$~\cite{CramaHammer} which is of course exponential in $n$, and so the claim follows.
\end{proof}

\begin{thm}
    \textbf{BNN} is incomparable with \textbf{CNF}.
\end{thm}
\begin{proof}
    We combine Lemma \ref{bnn_cnf} and Lemma \ref{cnf_bnn} and obtain the result.
\end{proof}

\begin{cor}
\textbf{BNN} is incomparable with \textbf{DNF}.
\end{cor}
\begin{proof}
    To show that \textbf{BNN} $\mathbf{\not\leq}$ \textbf{DNF} consider the function $\lnot f_n$ where $f_n$ is as defined in the proof of Lemma \ref{bnn_cnf}. After an application of DeMorgan laws, we obtain a DNF formula of linear size representing the function $\lnot f_n$. However, we may repeat the argument from Lemma \ref{bnn_cnf} for the negative vectors of $\lnot f$ which are all isolated and conclude that an exponential number of negative prototypes is required. To show that \textbf{DNF} $\mathbf{\not\leq}$ \textbf{BNN} it again suffices to consider the majority function. Any DNF of this monotone function must contain at least as many clauses as is the number of minimal true vectors~\cite{CramaHammer} which is of course again exponential in $n$.
\end{proof}

The last two relations for \textbf{BNN} already follow for free. 

\begin{cor}
    \textbf{BNN} is incomparable with both \textbf{IP} and \textbf{PI}.
\end{cor}
\begin{proof}
    Consider the function $f_n$ and its CNF $F_n$ from the proof of Lemma \ref{bnn_cnf}. Notice that not two clauses in $F_n$ are resolvable. Hence, $F_n$ is the IP representation of $f_n$ (contains exactly all prime implicates of $f_n$), and \textbf{BNN} $\mathbf{\not\leq}$ \textbf{PI} follows. The opposite relation \textbf{PI} $\mathbf{\not\leq}$ \textbf{BNN} follows from the fact that \textbf{CNN} $\mathbf{\not\leq}$ \textbf{BNN} because the \textbf{PI} language is a subset of the \textbf{CNF} language.

    By symmetric arguments, incomparability to \textbf{BNN} can be shown also for the \textbf{IP} language.
\end{proof}

We finish this section by a succinctness relation which does not appear in Figure~\ref{fig: succ} because it is just one half of an incomparability result between \textbf{BNN} and the \textbf{OBDD} language. 

\begin{prop}\label{thm: obdd}
    \textbf{BNN} $\mathbf{\not\leq}$ \textbf{OBDD}. 
\end{prop}
\begin{proof}
    It suffices to show that there is a Boolean function in $n$ variables with OBDD representation of polynomial size in $n$ which cannot be represented by a BNN of polynomial size in $n$. This property is satisfied by the threshold function $TH_n^{\floor{n/3}}$ which is true if and only if at least one third of its inputs are set to one. This is a symmetric Boolean function and it is long known that every such function can be represented by an OBDD of polynomial size in the number of variables~\cite{wegBook}. On the other hand, $TH_n^{\floor{n/3}}$ can only be represented by a BNN of exponential size in $n$ by Theorem~\ref{thm: bnn}.
\end{proof}

Obviously, the above result also implies \textbf{BNN} $\mathbf{\not\leq}$ \textbf{FBDD} since the \textbf{OBDD} language is a subset of the \textbf{FBDD} language. We conjecture that also \textbf{FBDD} $\mathbf{\not\leq}$ \textbf{BNN} holds (which would imply \textbf{OBDD} $\mathbf{\not\leq}$ \textbf{BNN}) but finding a family of functions necessary for such a statement remains an open problem.

\section{Transformations}

In this section we shall show that the \textbf{BNN} language unfortunately does not support any standard transformation from~\cite{KCM} except of negation, which is a trivial transformation for \textbf{BNN}, and singleton forgetting, for which the complexity status remains open.

\begin{obs}\label{obs: neg}
        \textbf{BNN} supports \textbf{}$\mathbf{\lnot}$\textbf{C}.
\end{obs}
\begin{proof}
    Since no ties are allowed, i.e. for every $x \in \mathcal{B}^n$ all prototypes nearest to $x$ must be of the same type, then it follows that $f$ is represented by BNN $(P,N)$ if and only if $\lnot f$ is represented by BNN $(N,P)$. If $(P,N)$ is a well defined BNN then so is $(N,P)$ and it can be of course obtained from $(P,N)$ in polynomial time.
\end{proof}

The fact that \textbf{BNN} does not support \textbf{CD} and \textbf{FO} can be proved using the exponential lower bound for threshold functions from Theorem \ref{thm: bnn}.

\begin{thm}\label{thm: cd}
    \textbf{BNN} does not support \textbf{CD}.
\end{thm}
\begin{proof}
    Let $(P,N)$ be the smallest $BNN$ representing the Boolean majority function on $n=4k$ variables, for some $k\in\mathbb{N}$. That is, $(P,N)$ represents the function $TH_{4k}^{2k}$ and $|(P,N)|\le\frac{n}{2}+2=2(k+1)$ holds by Theorem \ref{thm: bnn}. Let $x_1,\dots,x_n$ denote the variables of the threshold function. Notice that by setting some variable $x_i$ to $1$, i.e. conditioning on term $T=x_i$ we obtain a threshold function that has one variable fewer and threshold smaller by one.

    Let $T_k=x_1\land x_2\land\dots\land x_{k}$ be a consistent term. Then $(TH_{4k}^{2k}\mid T_k) = TH_{3k}^{k}$. It follows from Theorem \ref{thm: bnn} that in order to represent such a function, we need a $BNN$ of size $2^{\Omega(3k)}=2^{\Omega(n)}$, and thus cannot produce it in polynomial time from the input $(P,N)$ of size $O(n)$.
\end{proof}

The following lemma shows that for threshold functions conditioning (by a positive term) and forgetting work the same way producing the same result.

\begin{lemma}\label{lem: equiv}
    For $m,n\in\mathbb{N}$, $m\leq n$ consider the threshold function $f=TH^m_n$ and let $i\in\set{1,2,\dots,n}$ be arbitrary. Then 
    \[f|x_i\equiv\exists x_i. f(x_1,x_2,\dots,x_n)\,.\]
\end{lemma}
\begin{proof}
    From the proof of Theorem \ref{thm: cd}, we have that $f|x_i\equiv TH_{n-1}^{m-1}$. Thus it suffices to show that also $\exists x_i .f(x_1,x_2,\dots,x_n)\equiv TH_{n-1}^{m-1}$. By definition: 
	       \begin{gather*}
                \exists x_i.f(x_1,x_2,\dots,x_n)\equiv \\
                f(x_1,x_2,\dots,x_{i-1},0,x_{i+1},\dots,x_n)\lor f(x_1,x_2,\dots,x_{i-1},1,x_{i+1},\dots,x_n)\equiv\\
                TH_{n-1}^m\lor TH_{n-1}^{m-1}
            \end{gather*}
    Notice now that models of $TH_{n-1}^{m}$ form a subset of models of $TH_{n-1}^{m-1}$. We may then omit $TH_{n-1}^{m}$ and hence $\exists x_i.f(x_1,x_2,\dots,x_n)\equiv TH_{n-1}^{m-1}$, as desired.
\end{proof}

\begin{thm}\label{thm: fo}
    \textbf{BNN} does not support \textbf{FO}.
\end{thm}
\begin{proof}
    We combine Theorem \ref{thm: cd} and Lemma \ref{lem: equiv}. Again, $TH_n^{n/2}$ can be represented by a BNN of size $O(n)$, but we need a BNN of size $2^{\Omega(n)}$ after forgetting the first $n/4$ variables.
\end{proof}

Finally, we shall show that the \textbf{BNN} language does not support conjunction and disjunction even in the bounded case.

\begin{thm}\label{thm: conj}
    \textbf{BNN} does not support $\mathbf{\wedge BC}$.
\end{thm}
\begin{proof}
Consider the (non-strict) majority function on $2n$ variables 
\[  
(f(x_1, \ldots ,x_{2n})=1) \;\; \equiv \;\; (\sum_{i=1}^{2n} \geq n).
\]
This function has a BNN representation of size $O(n)$ by Theorem \ref{thm: bnn}. Furthermore, consider the (non-strict) minority function on $2n$ variables 
\[  
(g(x_1, \ldots ,x_{2n})=1) \;\; \equiv \;\; (\sum_{i=1}^{2n} \leq n).
\]
By a symmetric argument (just exchange $0$ and $1$ in the statement and proof of Theorem \ref{thm: bnn}) this function again admits a BNN representation of size $O(n)$. Now consider the function $f \wedge g$
\[  
((f \wedge g)(x_1, \ldots ,x_{2n})=1) \;\; \equiv \;\; (\sum_{i=1}^{2n} = n).
\]
This function has exactly $\binom{2n}{n} = 2^{\Omega(n)}$ isolated models (flipping a single bit in any model produces a vector where the sum of coordinates is either $n-1$ or $n+1$ which is in both cases a non-model). It follows by Corollary~\ref{cor: pos} that any BNN representation of $f \wedge g$ has size $2^{\Omega(n)}$.
\end{proof}

\begin{cor}\label{thm: disj}
    \textbf{BNN} does not support $\mathbf{\vee BC}$.
\end{cor}
\begin{proof}
Consider the negations of functions $f$ and $g$ from the previous proof (those are in fact strict minority and strict majority on $2n$ variables). Since negation preserves the size of BNN representations by Observation~\ref{obs: neg}, both $\lnot f$ and $\lnot g$ have BNN representation of size $O(n)$. However, $(\lnot f \vee \lnot g) = \lnot (f \wedge g)$ and any BNN representation of this function has size $2^{\Omega(n)}$ by the previous proof and the fact that negation preserves the size. 
\end{proof} 

The above results show, that when it comes to transformations, the \textbf{BNN} language is not a very good choice for a target compilation language. What disqualifies \textbf{BNN} the most is (in our opinion) the fact that it does not support conditioning unlike all other knowledge representation languages considered in \cite{KCM}, \cite{pbc}, and \cite{SL}. Conditioning is an essential transformation that is needed in many applications. In particular, if some of the variables are observable (and the others are decision variables) then queries are often asked after some (or all) of the observable variables are fixed to the observed values (which amounts to conditioning on this set of variables and only then answering the query).

The complexity status of all standard transformations for \textbf{BNN} and selected standard languages is summarized in Table~\ref{tab:transformations}. The only standard transformation for the \textbf{BNN} language for which the complexity status remains open is singleton forgetting. 

\begin{table}[htb]
    \centering
        \begin{tabular}{c | c || c | c || c | c || c | c || c}
             $\mathbf{L}$ & $\mathbf{CD}$ & $\mathbf{FO}$ & $\mathbf{SFO}$ & $\mathbf{\wedge C}$ & $\mathbf{\wedge BC}$ & $\mathbf{\vee C}$ & $\mathbf{\vee BC}$ & $\mathbf{\neg C}$ \\\hline\hline
             $\mathbf{NNF}$ & \checkmark  & $\circ$ &  \checkmark  & \checkmark & \checkmark & \checkmark & \checkmark &  \checkmark \\
             $\mathbf{d}$-$\mathbf{NNF}$ & \checkmark & $\circ$ & \checkmark & \checkmark & \checkmark & \checkmark & \checkmark & \checkmark \\
             $\mathbf{DNNF}$ &  \checkmark  & \checkmark &  \checkmark  & $\circ$ & $\circ$ & \checkmark & \checkmark &  $\circ$  \\
             $\mathbf{d}$-$\mathbf{DNNF}$ & \checkmark  & $\circ$ &  $\circ$  & $\circ$ & $\circ$ & $\circ$ & $\circ$ & ? \\
             $\mathbf{BDD}$ & \checkmark & $\circ$ & \checkmark & \checkmark & \checkmark & \checkmark & \checkmark & \checkmark \\
             $\mathbf{FBDD}$ & \checkmark  & $\bullet$ &  $\circ$  & $\bullet$ & $\circ$ & $\bullet$ & $\circ$ & \checkmark  \\
             $\mathbf{OBDD}$ & \checkmark  & $\bullet$ &  \checkmark  & $\bullet$ & $\circ$ & $\bullet$ & $\circ$ & \checkmark  \\
             $\mathbf{CNF}$ & \checkmark  & $\circ$ &  \checkmark  & \checkmark & \checkmark & $\bullet$ & \checkmark & $\bullet$  \\
             $\mathbf{DNF}$ & \checkmark  & \checkmark &  \checkmark  & $\bullet$ & \checkmark & \checkmark & \checkmark & $\bullet$  \\
             $\mathbf{PI}$ & \checkmark  & \checkmark &  \checkmark  & $\bullet$ & $\bullet$ & $\bullet$ & \checkmark & $\bullet$ \\
             $\mathbf{IP}$ & \checkmark  & $\bullet$ &  $\bullet$  & $\bullet$ & \checkmark & $\bullet$ & $\bullet$ &  $\bullet$ \\
             $\mathbf{MODS}$ & \checkmark  & \checkmark &  \checkmark  & $\bullet$ & \checkmark & $\bullet$ & $\bullet$ &  $\bullet$ \\\hline
             $\mathbf{CARD}$ & \checkmark  & $\circ$ & ? & \checkmark & \checkmark & $\bullet$ & $\bullet$ & $\bullet$ \\ 
             $\mathbf{PBC}$ & \checkmark  & $\bullet$ & $\bullet$  & \checkmark & \checkmark & $\bullet$ & $\bullet$ & $\bullet$ \\
             $\mathbf{SL}$ & \checkmark  & \checkmark &  \checkmark  & $\bullet$ & $\bullet$ & $\bullet$ & $\bullet$ &  \checkmark  \\ \hline
             $\mathbf{BNN}$ & $\bullet$  & $\bullet$ & ? & $\bullet$ & $\bullet$ & $\bullet$ &$\bullet$ &  \checkmark  \\
        \end{tabular}
    \caption{Languages introduced in \cite{KCM}, \cite{pbc}, and \cite{SL} and their polynomial time transformations. \checkmark means ``satisfies'', $\bullet$ means ``does not satisfy'', $\circ$ means ``does not satisfy unless P=NP'', and ? corresponds to an open problem.}
    \label{tab:transformations}
\end{table}

In this context, we note that also singleton conditioning is a transformation with an unknown complexity status. Of course, singleton conditioning was not considered as a standard transformation in~\cite{KCM} as all languages considered there satisfy even general conditioning, but it does make sense for the \textbf{BNN} language. We conjecture that both of these transformations can be done in polynomial time. Proving this conjecture would partly justify the usefulness of the \textbf{BNN} language, as conditioning on a constant size set of observable variables could (provably) lead to only a polynomial blowup of the resulting BNN representation (the current proof of hardness for general conditioning requires $\Omega(n)$ size set of variables).

\section{Queries}

In this section, we shall show that the \textbf{BNN} language supports a reasonably large subset of standard queries from~\cite{KCM}.
We begin with a trivial observation that both consistency and validity can be checked in constant time. 

\begin{obs}
        \textbf{BNN} supports \textbf{CO} and \textbf{VA}.
\end{obs}
\begin{proof}
    For a BNN representation $(P,N)$ consistency check is equivalent to checking that $P$ is non-empty while validity check is equivalent to checking that $N$ is empty. Both of these checks can be done in constant time. 
\end{proof}

Next, we prove that implicant check is supported by the \textbf{BNN} language.

\begin{thm}\label{thm: im}
    \textbf{BNN} supports \textbf{IM}.
\end{thm}

\begin{proof}
    Let $f$ be a Boolean function and $(P,N)$ its BNN representation. Let $T=l_1\land\dots\land l_k$ be a consistent term. Without loss of generality, we may assume that $\forall1\le i\le k:$ $l_i\in\set{x_i,\lnot x_i}$, since otherwise we may relabel the variables. Our aim is to design a polynomial time algorithm that checks whether $T\implies f$. Let us denote by $y\in\mathcal{B}^{k}$  the (partial) vector satisfying $T$, i.e. defined by $y_i=1$ if $l_i=x_i$ and $y_i=0$ if $l_i=\lnot x_i$. Furthermore, let $S$ denote the sub-cube of $\mathcal{B}^n$ determined by the vector $y$. That is, 
    \[
    S:=\set{x\in\mathcal{B}^n\mid \forall1\le i\le k:x_i=y_i}.
    \]
    Clearly, $T\implies f$ if and only if there is no negative vector of $f$ inside $S$ (i.e. after fixing the values in $y$ the resulting function is constant 1). We shall show that this condition can be tested efficiently.
    
    If there is a negative prototype in $S$ (which is easy to check), we are done and $T$ is not an implicate of $f$. If all negative prototypes are outside of $S$, let us denote by $N'$ the set of projections of all negative prototypes into $S$:
    \[
    N':= \set{\text{proj}_S(q) = (y_1, \ldots ,y_k,q_{k+1}, \ldots ,q_n) \mid q\in N},
    \]
    i.e. $\text{proj}_S(q)_i=y_i$ for $1\leq i\leq k$ and $\text{proj}_S(q)_i=q_i$ otherwise. 

    If there exists a negative prototype $q$ which is at most as far from its projection $\text{proj}_S(q)$ than any positive prototype, i.e.
    \[
    \exists q \in N \forall p \in P: d_H(p,\text{proj}_S(q)) \geq d_H(q,\text{proj}_S(q))
    \]
    then either there exists another negative prototype which is strictly closer to $\text{proj}_S(q)$ than $q$, or $q$ is the closest negative prototype to $\text{proj}_S(q)$ in which case the above inequality must be strict for every positive prototype by the definition of BNN representation. In both cases    
    $f(\text{proj}_S(q))=0$ follows, we have found a negative vector of $f$ in $S$, and thus we can again conclude that $T$ is not an implicate of $f$. Note, that this condition can be tested in polynomial time with respect to the size of $P$ for any fixed $q$ and thus in polynomial time with respect to the size of $(P,N)$ for all $q$.

    Let us now assume the opposite, namely that for every $\text{proj}_S(q)\in N'$ there exists a positive prototype $p\in P$ which is strictly closer to it than $q$:
    \[
    \forall q \in N \exists p \in P: d_H(p,\text{proj}_S(q)) < d_H(q,\text{proj}_S(q))
    \]
    We claim that in this case there are no negative vectors of $f$ inside $S$, and we can conclude that $T$ is an implicate of $f$.
    Let us assume for contradiction that there exists such a vector $x\in S$ for which $f(x)=0$. Let $q$ be the closest negative prototype to $x$ and let $q'=\text{proj}_S(q)$ be its projection on $S$. Furthermore, let $p$ be the positive prototype closest to $q'$. Then 
    \[
    d_H(q,x)=d_H(q,q')+d_H(q',x)>d_H(p,q')+d_H(q',x)\ge d_H(p,x),
    \]
    The first equality holds because $d_H(q,q')$ depends only on the first $k$ coordinates while $d_H(q',x)$ depends only on the remaining coordinates. The first inequality follows from the assumption and the second from the triangle inequality for Hamming distance. However, $d_H(q,x) > d_H(p,x)$ implies $f(x)=1$ which is a contradiction. 

    The above discussion is summarized in Algorithm \ref{alg: im}. The algorithm runs in time $O(n|P||N|)$ (which is polynomial in the size of the input $n(|P|+|N|)$), as its main part consists of two nested \emph{for} loops. In the worst case, the algorithm considers all projections of negative prototypes and calculates their distances to each of the positive prototypes.
\end{proof}
\begin{algorithm}
\caption{Implicant check for a BNN representation}\label{alg: im}
    \begin{algorithmic}
        \Require $(P,N)$ representing $f:\mathcal{B}^n\to\mathcal{B}$ and a consistent term $T=l_1\land\dots\land l_k$
        \Ensure \textbf{IM}$(f,T)$
        
        \For{$q\in N$}
            \If{$q \in S$}
                \State \Return 0
            \EndIf
            \State $q' \gets \text{proj}_S(q)$
            \State $d \gets +\infty$
            \For{$p\in P$}
                \State $d\gets\min\set{d,d_H(p,q')}$
            \EndFor
            \If{$d_H(q,q') \leq d$}
                \State \Return 0
            \EndIf
        \EndFor
        \State \Return 1
    \end{algorithmic}
\end{algorithm}

Since negation is trivial for the \textbf{BNN} language by Observation~\ref{obs: neg}, we immediately get the following result for clausal entailment. 

\begin{cor}
    \textbf{BNN} supports \textbf{CE}.
\end{cor}
\begin{proof}
     Let $f$ be a Boolean function and $(P,N)$ its BNN representation. Let $C=l_1\lor\dots\lor l_k$ be a consistent clause. Our aim is to design a polynomial time algorithm that checks whether $f\implies C$. Clearly $(f\implies C)\equiv(\lnot C\implies \lnot f)$. Moreover, $\lnot f$ is readily available (recall Observation~\ref{obs: neg}), and by DeMorgan laws the negation of a clause $C$ is a term $T$. So let $T=\lnot C$. Now, we may check whether $f$ entails $C$ by calling Algorithm \ref{alg: im} for inputs $\lnot f=(N,P)$ and $\lnot C=T$, getting the correct answer in polynomial time.
\end{proof}

It is interesting to note that for most standard languages which support CE this property stems from the fact that such languages support CD and CO. The CE algorithm in these cases first performs the required conditioning, then tests consistency, and then outputs yes if and only if the partial function is not consistent (i.e. identically zero). This is not the case for BNN as it supports CE despite of not supporting CD which is a rather unique combination of supported transformations and queries. We are not aware of any other knowledge representation language that supports clausal entailment without supporting conditioning.

The fact that \textbf{BNN} supports clausal entailment also directly implies that \textbf{BNN} supports model enumeration.

\begin{cor}
    \textbf{BNN} supports \textbf{ME}.
\end{cor}
\begin{proof}
Models can be enumerated using a tree search such as the simple DPLL algorithm where at every step before a value is assigned to a variable and a branch to a node on the next level is built, it is first tested whether the current partial assignment of values to variables plus the considered assignment yields a subfunction which is identically zero (and if yes then the branch is not built). This of course amounts to a clausal entailment test. Therefore all branches of the tree that is built have full length (all variables are fixed to constants) and terminate at models. Thus both the size of the tree and the total work required is upper bounded by a polynomial in the number of models.
\end{proof}

In the rest of this section, we shall prove that \textbf{BNN} does not support the remaining standard queries, i.e. that there are no polynomial time algorithms for EQ, SE and CT queries unless P=NP. All three proofs are based on the same reduction from the following NP-complete decision problem:

\medskip

\textbf{Half-Size Independent Set (HSIS)}\\
Input: An undirected graph $G=(V,E)$ with $n$ vertices.\\
Question: Does there exist an independent set with exactly $n/2$ vertices in $G$?

\medskip

Although IS, the general independent set problem (in which a parameter $k$ is part of the input and the question asks for the existence of an independent set of size $k$), is widely known to be NP-complete, we have to argue that also the restricted version HSIS is hard. To see this, consider the textbook reduction from 3-SAT to IS which for every cubic clause creates a clique of size three and then connects vertices that correspond to complementary literals (one edge for each such pair). It is easy to see that the input 3-SAT instance (with $m$ clauses) is satisfiable if and only if the constructed graph on $3m$ vertices contains an independent set of size exactly $m$. Modifying this construction by adding $m$ isolated vertices yields a reduction in which the input 3-SAT instance is satisfiable if and only if the constructed graph on $4m$ vertices contains an independent set of size exactly $2m$, which is an instance of HSIS. 

Now we are ready to prove the first hardness result.

\begin{thm}\label{thm: eq}
    The equivalence query is co-NP complete for the \textbf{BNN} language.
\end{thm}

\begin{proof}
It is easy to see that the EQ query for \textbf{BNN} belongs to co-NP. A certificate for a negative answer is a vector on which the two input BNN representations give opposite function values (which can be of course checked in polynomial time). 

For the hardness part, let $G=(V,E)$ with $n=2k$ vertices be an instance of HSIS (we assume without loss of generality that $k > 2$). We define two input BNN representations for the EQ query as follows:
\begin{enumerate}
    \item $F=(P_f,N_f)$ is a BNN representation of the majority function $f = MAJ_{2k}$ on $2k$ variables. We assume here that $F$ is the representation of $f$ from the proof of Theorem~\ref{thm: bnn}, namely 
    \begin{itemize}
        \item $P_f = \{x \in \mathcal{B}^{2k} \mid |x|=2k-1\}$, i.e all\footnote{The proof of Theorem~\ref{thm: bnn} in fact uses only arbitrarily chosen $k+1$ vectors of weight $2k-1$, however, we do not need a minimal representation here and taking all vectors of weight $2k-1$ does not change the represented function.} vectors of weight $2k-1$, and
        \item $N_f = \{(0,0, \ldots ,0)\}$, i.e. the only vector of weight $0$.
    \end{itemize}
    \item $H=(P_h,N_h)$ is a BNN representation of function $h$ defined as follows
    \begin{itemize}
        \item $P_h = \{p\} \cup \{p^e \mid e \in E \}$, where $p = (1,1, \ldots ,1)$ is the only vector of weight $2k$ and for every $e=(i,j) \in E$ the vector $p^e$ has weight $2k-2$ with $p^e_i = p^e_j = 0$ (all remaining $2k-2$ bits in $p^e$ are $1$).
        \item $N_h = \{x \in \mathcal{B}^{2k} \mid |x|=1\}$, i.e all vectors of weight $1$. Let us denote these vectors by $n^1, \ldots ,n^{2k}$ with $n^i_i = 1$ for every $1 \leq i \leq 2k$ (all remaining $2k-1$ bits in $n^i$ are $0$).
    \end{itemize}
\end{enumerate}
First, we shall show that $h(x)=f(x)$ holds for all vectors with weight different from $k$.
\begin{itemize}
    \item Let $x \in \mathcal{B}^{2k}$ be an arbitrary vector with $|x| \leq k-1$. Clearly, if we pick any index $i$ with $x_i = 1$ then the negative prototype $n^i$ of weight $1$ satisfies $d_H(x,n^i) \leq k-2$ (and if $x$ is the all-zero vector then  $d_H(x,n^i) = 1 \leq k-2$ for every $i$ because $k>2$ was assumed). On the other hand, all positive prototypes are at a Hamming distance at least $k-1$ from $x$ since at least that many zero-bits in $x$ have to be flipped to get a vector of weight $2k-2$ or more. Thus $h(x)=0=f(x)$.
    \item Let $x \in \mathcal{B}^{2k}$ be an arbitrary vector with $|x| \geq k+1$. In this case, we have $d_H(x,p) \leq k-1$ while all negative prototypes are at a Hamming distance at least $k$ from $x$ since at least that many one-bits in $x$ have to be flipped to get a vector of weight $1$. Thus $h(x)=1=f(x)$.
\end{itemize}
Thus $h(x)=f(x)$ holds everywhere except of the middle level $|x|=k$ of the Boolean lattice which we shall investigate now. 
\begin{itemize}
    \item Assume that there exists no independent set of size $k$ in $G$ and let $x \in \mathcal{B}^{2k}$ be an arbitrary vector with $|x|=k$. Vector $x$ has exactly $k$ coordinates with $x_i = 0$ and by our assumption, this index set cannot be an independent set of $G$. Thus there must exist $e=(i,j) \in E$ such that $x_i = x_j = 0$ and hence $d_H(x,p^e) = k-2$ since only the remaining $k-2$ zero-bits in $x$ have to be flipped to arrive to $p^e$. On the other hand, all negative prototypes are at a Hamming distance at least $k-1$ from $x$ since at least that many one-bits in $x$ have to be flipped to get a vector of weight $1$. Thus $h(x)=1=f(x)$ for every vector $x$ of weight $k$.
    \item Now assume the opposite, i.e. that there exists an independent set of size $k$ in $G$, defined by an index set $S$. Let $x^S \in \mathcal{B}^{2k}$ be a vector with $|x^S|=k$ where $S$ defines the zero-bits of $x^S$. Clearly, $d_H(x^S,p)=k$ but also $d_H(x^S,p^e) \geq k$ for every $e \in E$. To see the latter, notice that if $|e\cap S|=1$ we need to flip $k-1$ zero-bits and one one-bit in $x^S$ to arrive to $p^e$, and if $|e\cap S|=0$ we need to flip all $k$ zero-bits and two one-bits in $x^S$ to arrive to $p^e$. On the other hand, if we pick any index $i$ with $x^S_i = 1$ then the negative prototype $n^i$ of weight $1$ satisfies $d_H(x^S,n^i) = k-1$. Thus $h(x^S)=0$ while $f(x^S)=1$.
\end{itemize}
If we summarize the above observations we get that $h \equiv f$ if and only if there exists no independent set of size $k$ in $G$, or in other words the answer to the EQ query on BNN representations $F$ and $H$ is yes if and only if the answer to the input HSIS instance $G$ is no. Since both $F$ and $H$ can be clearly constructed from $G$ in polynomial time, this finishes the proof of NP-harness of the EQ query. 
\end{proof}

Theorem~\ref{thm: eq} immediately gives us the following corollary.

\begin{cor}
    The sentential entailment query is co-NP complete for the \textbf{BNN} language.
\end{cor}

\begin{proof}
    It is again easy to see that the SE query for \textbf{BNN} belongs to co-NP. Given a query $F \models H$ a certificate for a negative answer is a vector $x$ such that $F$ classifies $x$ as a positive vector while $H$ classifies $x$ as a negative vector (which can be of course checked in polynomial time). 

    The hardness part is a direct consequence of Theorem~\ref{thm: eq} since $F \equiv H$ if and only if $F \models H$ and $H \models F$ so any EQ query can be answered by asking two SE queries on the same input BNN representations. 
\end{proof}

Finally, the construction in the proof of Theorem~\ref{thm: eq} also yields the following result.

\begin{cor}
    The model counting query is NP-hard for the \textbf{BNN} language.
\end{cor}

\begin{proof}
    Notice that it is easy to count the number of models for $MAJ_{2k}$. Clearly, the number of models can be counted is half of all vectors plus half of the middle level, that is $CT(MAJ_{2k}) = CT(F) = 2^{2k-1} + \frac{1}{2} \genfrac(){0pt}{1}{2k}{k}$.     
    
    It is also obvious, that $CT(H) = CT(F)$ if and only if $H \equiv F$ (recall that the models of $H$ always form a subset of models of $F$) and so the ability to compute $CT(H)$ in polynomial time would answer the EQ query for $F$ and $H$ and thus decide the input HSIS instance $G$. 
\end{proof}

Let us remark that the reduction in the proof of Theorem~\ref{thm: eq} is number preserving in the following sense: any two distinct independent sets of size $k$ in $G$ correspond to two distinct non-models of $h$ of weight $k$. Therefore $CT(F) - CT(H)$ exactly equals the number of independent sets of size $k$ in $G$. This means that if the counting version of the HSIS problem is \#P hard (which we do not know, but it is quite likely since the counting version of the general IS problem certainly is \#P hard) then also the CT query for BNN is \#P hard. 

The complexity status of all standard queries for \textbf{BNN} and selected standard languages is summarized in Table~\ref{tab:queries}.

\begin{table}[htb]
     \centering
        \begin{tabular}{c | c | c | c | c | c | c | c | c}
             $\mathbf{L}$ & $\mathbf{CO}$ & $\mathbf{VA}$ & $\mathbf{CE}$ & $\mathbf{IM}$ & $\mathbf{EQ}$ & $\mathbf{SE}$ & $\mathbf{CT}$ & $\mathbf{ME}$ \\\hline\hline
             $\mathbf{NNF}$ & $\circ$ & $\circ$ & $\circ$ & $\circ$ & $\circ$ & $\circ$ & $\circ$ & $\circ$ \\
             $\mathbf{d}$-$\mathbf{NNF}$ & $\circ$ & $\circ$ & $\circ$ & $\circ$ & $\circ$ & $\circ$ & $\circ$ & $\circ$ \\
             $\mathbf{DNNF}$ &  \checkmark  & $\circ$ &  \checkmark  & $\circ$ & $\circ$ & $\circ$ & $\circ$ &  \checkmark  \\       
             $\mathbf{d}$-$\mathbf{DNNF}$ & \checkmark  & \checkmark &  \checkmark  & \checkmark & ? & $\circ$ & \checkmark &  \checkmark  \\
             $\mathbf{BDD}$ & $\circ$ & $\circ$ & $\circ$ & $\circ$ & $\circ$ & $\circ$ & $\circ$ & $\circ$ \\
             $\mathbf{FBDD}$ & \checkmark  & \checkmark &  \checkmark  & \checkmark & ? & $\circ$ & \checkmark &  \checkmark  \\
             $\mathbf{OBDD}$ & \checkmark  & \checkmark &  \checkmark  & \checkmark & \checkmark & $\circ$ & \checkmark &  \checkmark  \\
             $\mathbf{CNF}$ & $\circ$ & \checkmark & $\circ$ & \checkmark & $\circ$ & $\circ$ & $\circ$ & $\circ$ \\
             $\mathbf{DNF}$ & \checkmark & $\circ$ & \checkmark & $\circ$ & $\circ$ & $\circ$ & $\circ$ & \checkmark \\
             $\mathbf{PI}$ & \checkmark  & \checkmark &  \checkmark  & \checkmark & \checkmark & \checkmark & $\circ$ &  \checkmark  \\
             $\mathbf{IP}$ & \checkmark  & \checkmark &  \checkmark  & \checkmark & \checkmark & \checkmark & $\circ$ &  \checkmark  \\
             $\mathbf{MODS}$ & \checkmark  & \checkmark &  \checkmark  & \checkmark & \checkmark & \checkmark & \checkmark &  \checkmark \\\hline
             $\mathbf{CARD}$ & $\circ$ & \checkmark  & $\circ$ & \checkmark  & $\circ$ & $\circ$ & $\circ$ & $\circ$ \\
             $\mathbf{PBC}$ & $\circ$ & \checkmark  & $\circ$ & \checkmark  & $\circ$ & $\circ$ & $\circ$ & $\circ$ \\ 
             $\mathbf{SL}$ & \checkmark  & \checkmark &  \checkmark  & \checkmark & \checkmark & \checkmark & \checkmark &  \checkmark  \\ \hline
             $\mathbf{BNN}$ & \checkmark  & \checkmark &  \checkmark  & \checkmark & $\circ$ & $\circ$ & $\circ$ & \checkmark  \\
        \end{tabular}
    \caption{Languages introduced in \cite{KCM}, \cite{pbc}, and \cite{SL} and their polynomial time queries. \checkmark means ``satisfies'', $\circ$ means ``does not satisfy unless P=NP'', and ? corresponds to an open problem.}
    \label{tab:queries}
 \end{table}

\section{Conclusions}

We have studied the properties of the \textbf{BNN} language introduced in~\cite{NNR} with respect to the Knowledge Compilation Map. We have established succinctness relations of this language to languages \textbf{BDD}, \textbf{CNF}, \textbf{DNF}, \textbf{PI}, \textbf{IP}, and \textbf{MODS}. The most interesting question that remains open are the relations of \textbf{BNN} to \textbf{OBDD} and \textbf{FBDD}. We conjecture that in both cases the languages are incomparable. Another open question is the succinctness relation of \textbf{BDD} to the languages added to the Knowledge Compilation Map in subsequent papers, namely to \textbf{CARD}, \textbf{PBC}, and \textbf{SL} languages. 

Next, we have studied the complexity status of standard transformations and queries for the \textbf{BNN} language. Although it supports a decent subset of queries in polynomial time (CO, VA, IM, CE, ME) and hence it passes the necessary condition for a target compilation language formulated in~\cite{KCM}\footnote{Here we refer to the following sentence from~\cite{KCM}: "For a language to qualify as a target compilation language we will require that it permits a polytime clausal entailment test."}, we feel that the lack of supported transformations (only negation is supported), and in particular the fact that conditioning is not supported in polynomial time, disqualifies the \textbf{BNN} from being a good target language for knowledge compilation. The open problem left for future research is the complexity status of the singleton forgetting transformation.

\section{Acknowledgement}
We are grateful to Petr Ku\v cera for several helpful suggestions and for the construction that is used in the proof of Theorem~\ref{thm: conj}.

\bibliographystyle{plain}

\bibliography{bibliography}
\end{document}